\documentclass{article}

\usepackage{microtype}
\usepackage{graphicx}
\usepackage{subcaption}
\usepackage{booktabs} 
\usepackage{multirow} 

\usepackage{hyperref}

\usepackage[preprint]{icml2026}

\usepackage{amsmath}
\usepackage{amssymb}
\usepackage{mathtools}
\usepackage{amsthm}
\usepackage{makecell}

\usepackage[capitalize,noabbrev]{cleveref}

\theoremstyle{plain}
\newtheorem{theorem}{Theorem}[section]

\newtheorem{corollary}[theorem]{Corollary}
\theoremstyle{definition}
\newtheorem{definition}[theorem]{Definition}

\theoremstyle{remark}
\newtheorem{remark}[theorem]{Remark}
\newcommand{\loss}{\mathcal{L}}

\newcommand{\model}{f}
\newcommand{\params}{\theta}
\newcommand{\Ex}{\mathbb{E}}
\newcommand{\Real}{\mathbb{R}}
\newcommand{\Prob}{\mathbb{P}}

\newcommand{\MaCS}{\textit{MaCS}}
\usepackage{xspace}
\usepackage{fontawesome5} 
\usepackage{xcolor}

\usepackage[textsize=tiny]{todonotes}

\icmltitlerunning{Margin and Consistency Supervision for Calibrated and Robust Vision Models}

\begin{document}

\twocolumn[
  \icmltitle{Margin and Consistency Supervision for Calibrated and Robust Vision \\Models}



  \icmlsetsymbol{equal}{*}

  \begin{icmlauthorlist}
    \icmlauthor{Salim Khazem}{yyy}
  \end{icmlauthorlist}

  \icmlaffiliation{yyy}{Talan Research Center, Paris, France}

  \icmlcorrespondingauthor{Salim Khazem}{salim.khazem@talan.com}

  \icmlkeywords{Machine Learning, ICML}

  \vskip 0.3in
]



\printAffiliationsAndNotice{}  

\begin{abstract}
Deep vision classifiers often achieve high accuracy while remaining poorly calibrated and fragile under small distribution shifts. We present Margin and Consistency Supervision ($\textit{MaCS}$), a simple, architecture-agnostic regularization framework that jointly enforces logit-space separation and local prediction stability. $\textit{MaCS}$ augments cross-entropy with (i) a hinge-squared margin penalty that enforces a target logit gap between the correct class and the strongest competitor, and (ii) a consistency regularizer that minimizes the KL divergence between predictions on clean inputs and mildly perturbed views. We provide a unifying theoretical analysis showing that increasing classification margin while reducing local sensitivity formalized via a Lipschitz-type stability proxy yields improved generalization guarantees and a provable robustness radius bound scaling with the margin-to-sensitivity ratio. Across several image classification benchmarks and several backbones spanning CNNs and Vision Transformers, $\textit{MaCS}$ consistently improves calibration (lower ECE and NLL) and robustness to common corruptions while preserving or improving top-1 accuracy. Our approach requires no additional data, no architectural changes, and negligible inference overhead, making it an effective drop-in replacement for standard training objectives. 
\end{abstract}

\begin{figure}[t]
  \vskip 0.1in
  \begin{center}
    \centerline{\includegraphics[width=\columnwidth]{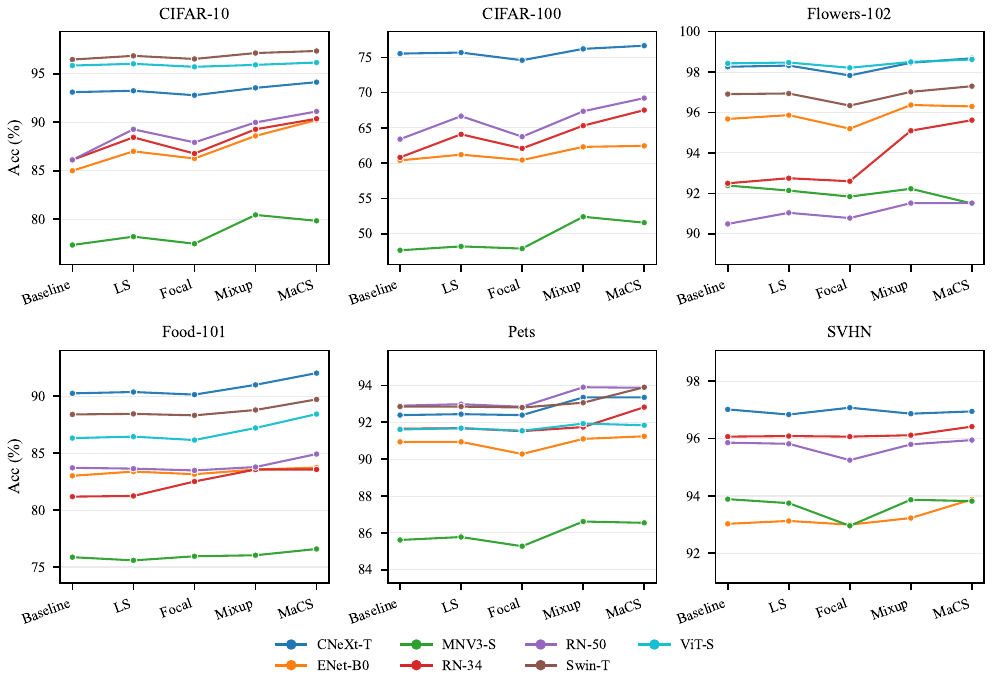}}
    \caption{Per-dataset model curves across methods. Each line corresponds to a model and traces accuracy across training objectives, highlighting method-specific gains.}
    \label{fig:model-curves}
  \end{center}
  \vskip -0.2in
\end{figure}

\section{Introduction}
\label{sec:intro}

Deep neural networks achieve impressive accuracy on vision classification benchmarks, yet they frequently exhibit problematic behaviors: overconfident predictions on ambiguous or out-of-distribution inputs \citep{guo2017calibration}, brittleness to small input perturbations \citep{szegedy2014intriguing}, and poor generalization under distribution shift \citep{hendrycks2019benchmarking}. These issues limit deployment in safety-critical applications where reliable uncertainty quantification is essential.

Classical statistical learning theory provides a compelling lens for understanding these failures: classifiers with larger \emph{margins} the gap between the score for the correct class and the highest-scoring alternative generalize better \citep{bartlett1998sample, bartlett2017spectrally}. Intuitively, larger margins provide a ``buffer zone'' that makes predictions more robust to noise and distribution shift. Similarly, \emph{locally stable} classifiers that produce consistent predictions under small input changes are inherently more robust \citep{cohen2019certified}.

Motivated by these insights, we propose \textbf{Margin and Consistency Supervision} (\MaCS), a unified training objective that encourages both properties. \MaCS{} augments cross-entropy with two complementary terms:

\begin{enumerate}
    \item \textbf{Margin Loss:} A hinge-squared penalty encouraging the logit margin $\gamma(x) = f_y(x) - \max_{j \neq y} f_j(x)$ to exceed a target threshold $\Delta$, promoting well-separated representations.
    
    \item \textbf{Consistency Loss:} A KL-divergence term measuring prediction stability under mild perturbations (noise and blur), encouraging smooth decision boundaries.
\end{enumerate}

The combined objective is simply:
\begin{equation}
    \loss_{\MaCS} = \loss_{\text{CE}} + \lambda_m \loss_{\text{margin}} + \lambda_c \loss_{\text{cons}},
\end{equation}
which can be applied to any classifier without architectural changes.

\textbf{Contributions.} Our main contributions are: (i) We propose \MaCS, a simple, architecture-agnostic regularization framework combining margin maximization and consistency regularization. (ii) We provide a unifying theoretical analysis connecting margin and local sensitivity to generalization guarantees and a provable robustness radius bound. (iii) We conduct extensive experiments on 6 datasets and 7 architectures spanning CNNs and Vision Transformers, showing consistent improvements in calibration, robustness, and accuracy. We release a fully reproducible codebase requiring no additional data or architectural changes. Code is released and available at : ~\url{https://github.com/salimkhazem/marco}.

\section{Related Work}
\label{sec:related}

\textbf{Margin-Based Learning.}
The connection between large margins and generalization has deep roots in statistical learning theory \citep{vapnik1998statistical, bartlett1998sample}. For neural networks, margin-based objectives have been explored primarily in metric learning and face recognition \citep{liu2017sphereface, deng2019arcface, wang2018cosface}, where angular margins in embedding space improve discriminability. \citet{elsayed2018large} study margin maximization for general classification but focus on adversarial robustness rather than calibration. Label smoothing \citep{szegedy2016rethinking} and focal loss \citep{lin2017focal} implicitly affect margins but do not explicitly enforce a target gap.

\textbf{Consistency Regularization.}
Consistency-based methods have proven highly effective for semi-supervised learning \citep{berthelot2019mixmatch, sohn2020fixmatch}, where they enforce that predictions remain stable under data augmentation. Similar principles appear in self-training \citep{xie2020self}, domain adaptation \citep{french2018selfensembling}, and knowledge distillation \citep{hinton2015distilling}. Mixup \citep{zhang2018mixup} and its variants encourage smooth interpolations between training examples. We adapt consistency regularization for fully-supervised learning, using it to encourage smooth decision boundaries via explicit KL-divergence penalties.

\textbf{Robustness and Calibration.}
Adversarial training \citep{madry2018towards} and certified defenses \citep{cohen2019certified, lecuyer2019certified} address worst-case robustness but often sacrifice clean accuracy. Calibration methods include post-hoc temperature scaling \citep{guo2017calibration}, explicit calibration losses \citep{kumar2018trainable, mukhoti2020calibrating}, and mixup variants \citep{thulasidasan2019mixup}. \MaCS{} addresses accuracy, robustness, and calibration jointly through a single training objective without requiring separate post-hoc calibration.

\textbf{Consistency-Based Robustness.} Virtual Adversarial Training (VAT) \citep{miyato2018virtual} uses adversarial perturbations with KL consistency for semi-supervised learning. TRADES \citep{zhang2019theoretically} extends this to adversarial robustness by balancing clean and adversarial accuracy. Recent work DiGN \citep{tsiligkaridis2023diverse} applies Gaussian noise consistency in fully supervised settings with theoretical control over gradient norms. In contrast to worst-case (adversarial) perturbations, \MaCS{} uses mild perturbations to encourage local smoothness while simultaneously enforcing explicit logit-margin constraints, targeting calibration and corruption robustness rather than adversarial robustness.

\textbf{Data Augmentation for Robustness.} AugMix \citep{hendrycks2020augmix} combines diverse augmentation chains with Jensen-Shannon consistency and is a strong baseline for corruption robustness. Subsequent methods including DeepAugment
\citep{hendrycks2021many}, PixMix \citep{hendrycks2022pixmix}, and AugMax \citep{wang2021augmax} further improve robustness through sophisticated augmentation strategies. IPMix \citep{diffenderfer2024ipmix} extends these ideas with
improved augmentation pipelines. These methods are complementary to \MaCS{}: we show that combining \MaCS{} with AugMix yields additive improvements (Table~\ref{tab:robustness}), suggesting \MaCS{} can serve as a base layer for
augmentation-based robustness stacks. Notably, AugMix explicitly avoids overlap with CIFAR-C corruptions in its augmentation set, while \MaCS{} uses simple noise/blur; we analyze this overlap in Section~\ref{sec:robustness}.

\textbf{Certified and Lipschitz Methods.} Certified Radius Maximization (CRM) \citep{zhai2020macer} and Lipschitz-constrained networks \citep{tsuzuku2018lipschitz} provide formal guarantees by controlling model smoothness. While \MaCS{}
does not provide certified robustness, our theoretical framework (Section~\ref{sec:theory}) shares the intuition that controlling local sensitivity enables robustness; \MaCS{} trades formal guarantees for practicality and preserved
accuracy.

\section{Method}
Figure~\ref{fig:method} provides an overview of the \MaCS{} training pipeline.
\label{sec:method}
  \begin{figure}[t]
    \centering
    \includegraphics[width=\columnwidth]{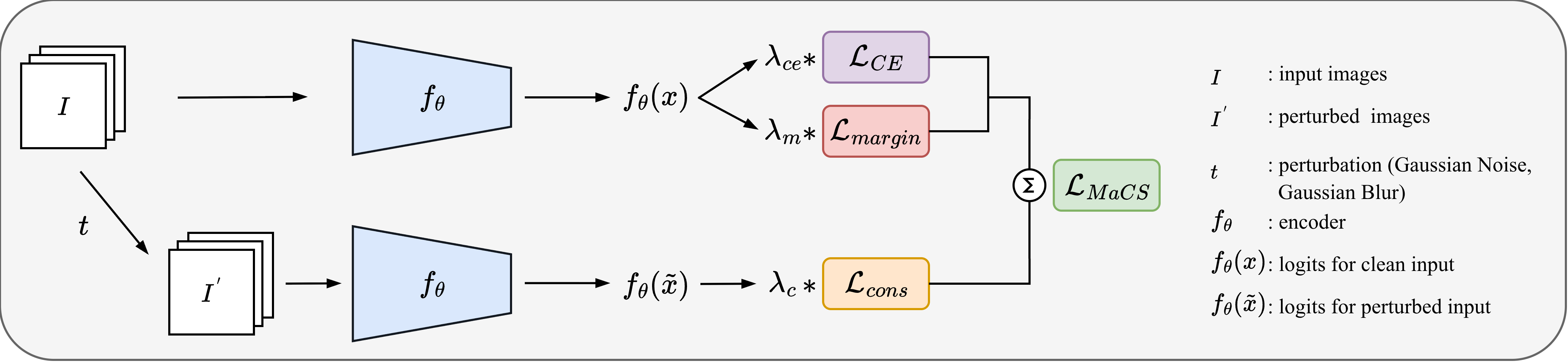}
    \caption{\textbf{Overview of \MaCS{} training.} The model processes both clean input $x$ and perturbed input $\tilde{x} = T(x)$. The total loss combines cross-entropy, a margin penalty encouraging $\gamma(x) \geq \Delta$, and a
  KL-based consistency term enforcing prediction stability.}
    \label{fig:method}
  \end{figure}
  
\subsection{Problem Setup}
Consider a classifier $\model(x; \params): \Real^d \to \Real^K$ mapping inputs to logits over $K$ classes. For a labeled example $(x, y)$ with $y \in \{1, \ldots, K\}$, the standard cross-entropy loss is:
\begin{equation}
    \loss_{\text{CE}}(x, y) = -\log \frac{\exp(f_y(x))}{\sum_{j=1}^K \exp(f_j(x))}.
\end{equation}

\subsection{Margin Loss}

The \emph{logit margin} for a sample $(x, y)$ is the gap between the correct class logit and the strongest competitor:
\begin{equation}
    \gamma(x) = f_y(x) - \max_{j \neq y} f_j(x).
\end{equation}
Positive margins indicate correct classification; larger margins suggest more confident and robust predictions.

We encourage margins to exceed a target $\Delta > 0$ using a hinge-squared penalty:
\begin{equation}
    \loss_{\text{margin}}(x, y) = \max(0, \Delta - \gamma(x))^2.
    \label{eq:margin_loss}
\end{equation}
The squared form provides smooth gradients and penalizes small margins more aggressively than a linear hinge. We use $\Delta = 1$ in all experiments.

\subsection{Consistency Loss}

To encourage local stability, we measure the KL divergence between predictions on clean and perturbed inputs:
\begin{equation}
    \loss_{\text{cons}}(x) = D_{\text{KL}}\big(p(x) \,\|\, p(\tilde{x})\big),
    \label{eq:cons_loss}
\end{equation}
where $p(x) = \text{softmax}(f(x))$ and $\tilde{x} = T(x)$ is a perturbed version of $x$.

The perturbation $T$ applies mild transformations that preserve semantic content:
\begin{itemize}
    \item Gaussian noise: $\tilde{x} = x + \epsilon$, $\epsilon \sim \mathcal{N}(0, \sigma^2 I)$ with $\sigma = 0.1$
    \item Gaussian blur: kernel size $3 \times 3$, $\sigma \in [0.1, 2.0]$
\end{itemize}

\subsection{Combined \MaCS{} Objective}

The complete \MaCS{} objective is:
\begin{equation}
    \loss_{\MaCS} = \loss_{\text{CE}} + \lambda_m \loss_{\text{margin}} + \lambda_c \loss_{\text{cons}},
    \label{eq:macs}
\end{equation}
where $\lambda_m, \lambda_c > 0$ control the regularization strength. We use $\lambda_m = 0.1$ and $\lambda_c = 0.5$ as defaults, tuned once on CIFAR-100 and fixed for all experiments.

\begin{algorithm}[t]
\caption{\MaCS{} Training}
\label{alg:macs}
\begin{algorithmic}[1]
\REQUIRE Model $f_\theta$, dataset $\mathcal{D}$, hyperparameters $\Delta, \lambda_m, \lambda_c$
\FOR{each mini-batch $(X, Y) \sim \mathcal{D}$}
    \STATE $Z \gets f_\theta(X)$ \COMMENT{Forward pass}
    \STATE $\loss_{\text{CE}} \gets \text{CrossEntropy}(Z, Y)$
    \STATE $\gamma \gets Z_Y - \max_{j \neq Y} Z_j$ \COMMENT{Margins}
    \STATE $\loss_{\text{margin}} \gets \text{mean}(\max(0, \Delta - \gamma)^2)$
    \STATE $\tilde{X} \gets T(X)$ \COMMENT{Perturb inputs}
    \STATE $\tilde{Z} \gets f_\theta(\tilde{X})$
    \STATE $\loss_{\text{cons}} \gets D_{\text{KL}}(\text{softmax}(Z) \| \text{softmax}(\tilde{Z}))$
    \STATE $\loss \gets \loss_{\text{CE}} + \lambda_m \loss_{\text{margin}} + \lambda_c \loss_{\text{cons}}$
    \STATE Update $\theta$ via gradient descent on $\loss$
\ENDFOR
\end{algorithmic}
\end{algorithm}

\section{Theoretical Motivation}
\label{sec:theory}

We provide theoretical grounding for \MaCS{} by connecting margin maximization and consistency regularization to generalization and robustness. Our analysis reveals that the \emph{margin-to-sensitivity ratio} governs certified robustness, motivating both components of our objective.

\subsection{Margin and Generalization}

Large margins have long been associated with better generalization \citep{bartlett1998sample, bartlett2017spectrally}. We recall a classical margin-based bound adapted to neural networks.

\begin{definition}[Spectral Complexity]
\label{def:spectral}
For a neural network $\model$ with $L$ layers and weight matrices $W_1, \ldots, W_L$, the spectral complexity is
\[
R_\model = \left(\prod_{\ell=1}^{L} \|W_\ell\|_2\right) \cdot \left(\sum_{\ell=1}^{L} \frac{\|W_\ell\|_F^{2/3}}{\|W_\ell\|_2^{2/3}}\right)^{3/2},
\]
where $\|W\|_2$ denotes the spectral norm and $\|W\|_F$ the Frobenius norm.
\end{definition}

\begin{theorem}[Margin-Based Generalization {\citep[Theorem 1.1]{bartlett2017spectrally}}]
\label{thm:margin}
Let $\model: \Real^d \to \Real^K$ be a neural network with spectral complexity $R_\model$, and let $\mathcal{D}$ be a distribution over $\Real^d \times [K]$. For any margin $\gamma > 0$, with probability at least $1 - \delta$ over an i.i.d.\ training set $S$ of size $n$ drawn from $\mathcal{D}$:
\[
\Prob_{(x,y) \sim \mathcal{D}}[\gamma(x) \leq 0] \leq \hat{L}_\gamma(S) + \tilde{O}\left(\frac{R_\model B}{\gamma \sqrt{n}}\right) + \sqrt{\frac{\log(1/\delta)}{2n}},
\]
where $\hat{L}_\gamma(S)$ is the fraction of training samples with margin less than $\gamma$, $B$ is a bound on the input norm $\|x\|_2 \leq B$, and $\tilde{O}$ hides logarithmic factors in $n$, $L$, and $K$.
\end{theorem}

This bound decreases with larger margins $\gamma$. By enforcing $\gamma(x) \geq \Delta$ on training samples, $\loss_{\text{margin}}$ directly reduces $\hat{L}_\gamma(S)$ for $\gamma \leq \Delta$, tightening the generalization bound.

\subsection{Consistency as a Sensitivity Proxy}

We formalize local sensitivity and connect it to the consistency loss.

\begin{definition}[Local Sensitivity]
\label{def:sensitivity}
For a classifier $\model$ with output probabilities $p(x) = \mathrm{softmax}(\model(x))$ and a perturbation distribution $\mathcal{P}$ over $\Real^d$, the local sensitivity at $x$ is
\[
S(x) = \Ex_{\epsilon \sim \mathcal{P}}\left[\|p(x) - p(x + \epsilon)\|_1\right].
\]
\end{definition}

\begin{remark}[Consistency Controls Sensitivity]
\label{rem:sensitivity}
For a single perturbation $\tilde{x} = x + \epsilon$, Pinsker's inequality gives
\[
\|p(x) - p(\tilde{x})\|_1 \leq \sqrt{2 D_{\mathrm{KL}}(p(x) \| p(\tilde{x}))} = \sqrt{2 \loss_{\mathrm{cons}}(x)}.
\]
Since our training procedure samples $\tilde{x} \sim T(x)$ and minimizes $\Ex[\loss_{\mathrm{cons}}]$, Jensen's inequality implies that minimizing the expected consistency loss reduces the local sensitivity $S(x)$.
\end{remark}

\subsection{Robustness via Margin-to-Sensitivity Ratio}

We now present our main theoretical result: the robustness radius is controlled by the ratio of margin to local Lipschitz constant. Throughout, $\|\cdot\|$ denotes the $\ell_2$ norm.

\begin{theorem}[Margin-Stability Robustness Radius]
\label{thm:robustness}
Fix an input-label pair $(x, y)$ with positive margin $\gamma(x) = f_y(x) - \max_{j \neq y} f_j(x) > 0$. For each $j \neq y$, define the logit-difference function $g_{y,j}(z) = f_y(z) - f_j(z)$. Suppose there exist $r > 0$ and $L_g(x) > 0$ such that for all $z$ with $\|z - x\| \leq r$,
\[
|g_{y,j}(z) - g_{y,j}(x)| \leq L_g(x) \|z - x\| \quad \forall j \neq y.
\]
Then the predicted class is invariant under all perturbations $\delta$ satisfying
\[
\|\delta\| < \min\left\{r,\; \frac{\gamma(x)}{L_g(x)}\right\}.
\]
\end{theorem}

\begin{proof}
Let $x' = x + \delta$ with $\|\delta\| < \gamma(x)/L_g(x)$. For any $j \neq y$:
\begin{align*}
g_{y,j}(x') &= g_{y,j}(x) + \bigl(g_{y,j}(x') - g_{y,j}(x)\bigr) \\
&\geq g_{y,j}(x) - |g_{y,j}(x') - g_{y,j}(x)| \\
&\geq g_{y,j}(x) - L_g(x)\|\delta\|.
\end{align*}
By definition of the margin, $g_{y,j}(x) = f_y(x) - f_j(x) \geq \gamma(x)$ for all $j \neq y$. Hence,
\[
g_{y,j}(x') \geq \gamma(x) - L_g(x)\|\delta\| > 0,
\]
where the strict inequality follows from $\|\delta\| < \gamma(x)/L_g(x)$. Thus $f_y(x') > f_j(x')$ for all $j \neq y$, so $\arg\max_k f_k(x') = y$.
\end{proof}

\begin{corollary}[Radius Under Lipschitz Logits]
\label{cor:linf-radius}
Suppose additionally that for some $L_f(x) > 0$,
\[
\|f(z) - f(x)\|_\infty \leq L_f(x) \|z - x\| \quad \text{for all } \|z - x\| \leq r.
\]
Then the predicted class is invariant for all perturbations with $\|\delta\| < \min\{r,\, \gamma(x)/(2L_f(x))\}$.
\end{corollary}

\begin{proof}
For any $j \neq y$ and $z$ with $\|z - x\| \leq r$:
\begin{align*}
|g_{y,j}(z) - g_{y,j}(x)| &= |(f_y(z) - f_j(z)) - (f_y(x) - f_j(x))| \\
&\leq |f_y(z) - f_y(x)| + |f_j(z) - f_j(x)| \\
&\leq 2\|f(z) - f(x)\|_\infty \\
&\leq 2L_f(x)\|z - x\|.
\end{align*}
Thus $L_g(x) \leq 2L_f(x)$, and the result follows from Theorem~\ref{thm:robustness}.
\end{proof}

\subsection{Connecting \MaCS{} to the Robustness Radius}

The robustness radius in Theorem~\ref{thm:robustness} scales with the ratio $\gamma(x)/L_g(x)$. \MaCS{} targets both quantities:

\begin{enumerate}
\item \textbf{Margin maximization:} The loss $\loss_{\mathrm{margin}}$ directly increases $\gamma(x)$ by penalizing samples with $\gamma(x) < \Delta$.

\item \textbf{Sensitivity reduction:} The loss $\loss_{\mathrm{cons}}$ penalizes prediction changes under perturbations in probability space.
\end{enumerate}

\begin{remark}[On the Theory-Practice Gap]
\label{rem:gap}
Theorem~\ref{thm:robustness} requires bounding $L_g(x)$, the Lipschitz constant of logit differences. The consistency loss $\loss_{\mathrm{cons}}$ operates on softmax probabilities rather than logits directly. While softmax is Lipschitz with respect to its inputs, the relationship between $D_{\mathrm{KL}}(p(x) \| p(\tilde{x}))$ and $L_g(x)$ is indirect: small KL divergence on sampled perturbations does not immediately imply a small worst-case Lipschitz constant.

We therefore treat $\loss_{\mathrm{cons}}$ as an \emph{empirical proxy} for local sensitivity. In Section~\ref{sec:analysis}, we verify this connection by measuring both the margin $\gamma(x)$ and a sensitivity estimate $\hat{L}(x) = \Ex_{\epsilon}[\|f(x+\epsilon) - f(x)\|_\infty / \|\epsilon\|]$ on held-out data, confirming that \MaCS{} improves the ratio $\gamma(x)/\hat{L}(x)$ compared to baseline training.
\end{remark}


\section{Experiments}
\label{sec:experiments}
We evaluate \MaCS{} across diverse datasets and architectures, focusing on three key metrics: accuracy, calibration (ECE, NLL), and robustness to common corruptions.

\subsection{Experimental Setup}

\textbf{Datasets:} We evaluate on six publicly available datasets spanning diverse domains.\textbf{CIFAR-10/100} \citep{krizhevsky2009learning} (32$\times$32 natural images with 10/100 classes), \textbf{SVHN} \citep{netzer2011reading} (street view house numbers, 10 digit classes), \textbf{Oxford Pets} \citep{parkhi2012cats} (37 pet breed categories), \textbf{Food-101} \citep{bossard2014food} (101 food categories), and \textbf{Flowers-102} \citep{nilsback2008automated} (102 flower species).

\textbf{Models.} We benchmark seven architectures spanning CNNs and Vision Transformers: \textbf{CNNs} include ResNets \citep{he2016deep}, ConvNeXt-Tiny \citep{liu2022convnet}, EfficientNet-B0 \citep{tan2019efficientnet}, and MobileNetV3-Small~\citep{howard2019searching}; \textbf{Transformers} include ViT-Small/16 \citep{dosovitskiy2021image} and Swin-Tiny \citep{liu2021swin}.

\textbf{Baselines.} We compare against: Cross-entropy (CE), Label Smoothing ($\epsilon=0.1$), Focal Loss ($\gamma=2$), and Mixup ($\alpha=0.2$).

\textbf{Training.} All models use mixed precision (AMP), cosine learning rate schedule with linear warmup, and ImageNet-pretrained initialization. \MaCS{} hyperparameters ($\Delta=1$, $\lambda_m=0.1$, $\lambda_c=0.5$) are tuned once on CIFAR-100 validation and fixed across all experiments. We report mean $\pm$ std over 3 random seeds.

\textbf{Metrics.} We evaluate: (1) Top-1 test accuracy, (2) Expected Calibration Error (ECE) with 15 bins \citep{naeini2015obtaining}, (3) Negative Log-Likelihood (NLL), and (4) mean accuracy under synthetic corruptions following \citet{hendrycks2019benchmarking}.

\subsection{Main Results}

\begin{table}[t]
\centering
\caption{Comparison on CIFAR-10/100 with ResNet-50. \MaCS{} achieves the best accuracy and robustness. Results: mean$_{\pm\text{std}}$ over 3 seeds.}
\label{tab:cifar_main}
\small
\setlength{\tabcolsep}{4pt}
\begin{tabular}{lccc}
\toprule
\textbf{Method} & \textbf{Accuracy} $\uparrow$ & \textbf{Robust} $\uparrow$ \\
\midrule
\multicolumn{3}{c}{\textit{CIFAR-10}} \\
\midrule
Baseline (CE) & 87.63$_{\pm 0.51}$ & 39.92$_{\pm 0.33}$ \\
Focal Loss & 87.35$_{\pm 0.67}$ & 40.94$_{\pm 0.95}$ \\
Label Smoothing & 88.85$_{\pm 0.54}$ & 41.87$_{\pm 0.53}$ \\
Mixup & 89.63$_{\pm 0.60}$ & 42.68$_{\pm 0.36}$ \\
\textbf{MaCS (Ours)} & \textbf{91.10$_{\pm 0.55}$} & \textbf{43.48$_{\pm 0.28}$} \\
\midrule
\multicolumn{3}{c}{\textit{CIFAR-100}} \\
\midrule
Baseline (CE) & 63.41$_{\pm 0.70}$ & 20.00$_{\pm 0.33}$ \\
Focal Loss & 62.93$_{\pm 0.98}$ & 20.85$_{\pm 0.86}$ \\
Label Smoothing & 65.38$_{\pm 1.44}$ & 22.26$_{\pm 0.58}$ \\
Mixup & 66.34$_{\pm 1.16}$ & 23.23$_{\pm 0.51}$ \\
\textbf{MaCS (Ours)} & \textbf{69.23$_{\pm 0.98}$} & \textbf{24.60$_{\pm 0.64}$} \\
\bottomrule
\end{tabular}
\end{table}

Table~\ref{tab:cifar_main} shows that \MaCS{} achieves the best  accuracy on both CIFAR-10 (91.10\%) and CIFAR-100 (69.23\%) with ResNet-50, outperforming all baseline methods. While Mixup achieves good corruption robustness, \MaCS{} provides a strong accuracy-robustness trade-off.
Figure~\ref{fig:model-curves} presents test accuracy across all datasets, seven models, and five methods. \MaCS{} achieves the highest accuracy in \textbf{27 out of 38} dataset-model configurations (71\%) and ties for best in 3 more, demonstrating broad effectiveness. The gains are particularly pronounced on CIFAR-10/100, Flowers-102, and Food-101, where \MaCS{} outperforms all baselines including Mixup. Mixup is the most frequent runner-up (5 wins and 3 ties). Figure~\ref{fig:model-curves} summarizes method profiles across models.

On CIFAR-10, \MaCS{} achieves the best accuracy across all architectures except MobileNetV3-Small, with notable improvements on EfficientNet-B0 ($+5.2$pp over baseline) and ResNet-50 ($+5.0$pp). On CIFAR-100, \MaCS{} yields substantial gains on ResNet-34 ($+6.7$pp) and ResNet-50 ($+5.8$pp). On Food-101, \MaCS{} is the best method on all seven architectures.

The only consistent failure case is MobileNetV3-Small, where Mixup tends to outperform \MaCS{}. We hypothesize that the small model capacity limits the ability to simultaneously optimize margin and consistency objectives; we discuss this further in Section~\ref{sec:analysis}.

\begin{figure}[h]
  \begin{center}
    \centerline{\includegraphics[width=\columnwidth]{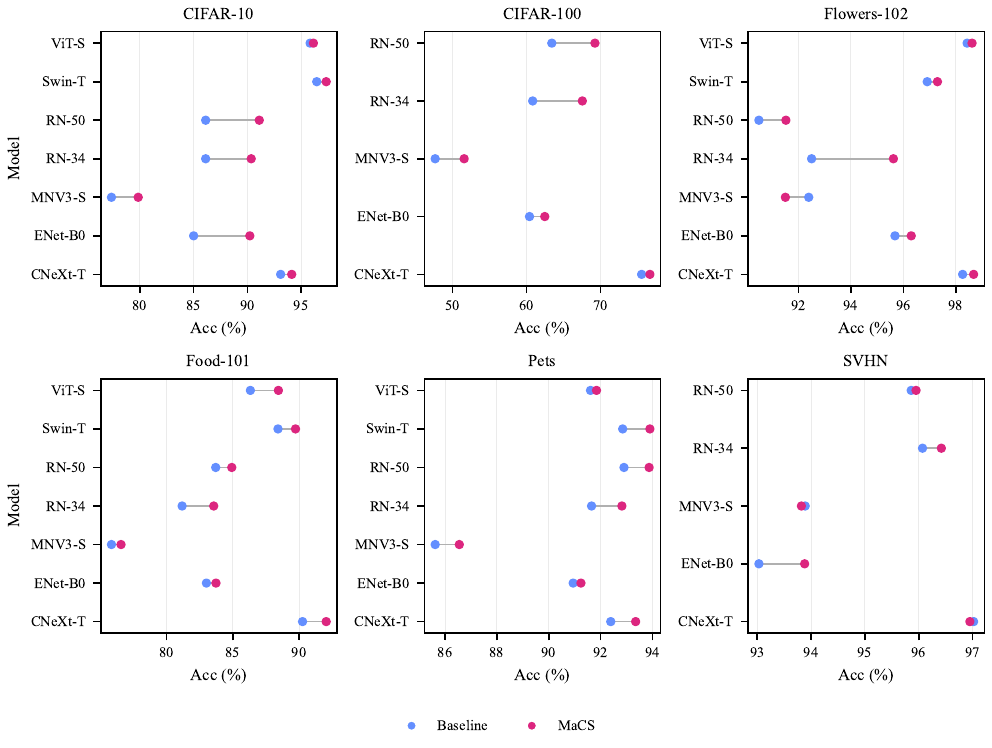}}
    \caption{Accuracy improvement of \MaCS{} over baseline (cross-entropy) across all dataset--model configurations. Each line connects the baseline (left) to \MaCS{} (right) accuracy. \MaCS{} improves over baseline in the large majority of settings, with the largest gains on CIFAR and Food-101.}
    \label{fig:dumbbell}
  \end{center}
\end{figure}

  \begin{table}[t]
  \caption{Perturbation ablation on CIFAR-100 (ResNet-50). Combining noise and blur yields the best robustness.}
  \label{tab:perturbation}
  \centering
  \small
  \begin{tabular}{lccc}
  \toprule
  Perturbation & Acc $\uparrow$ & ECE $\downarrow$ & Robust $\uparrow$ \\
  \midrule
  None (CE only) & 63.41 & 24.57 & 20.00 \\
  Noise only & 67.82 & 4.21 & 23.14 \\
  Blur only & 68.15 & 3.89 & 23.52 \\
  \textbf{Noise + Blur} & \textbf{69.23} & \textbf{3.13} & \textbf{24.60} \\
  \bottomrule
  \end{tabular}
  \end{table}

\subsection{Calibration Analysis}
A key benefit of \MaCS{} is improved calibration without post-hoc adjustment. Table~\ref{tab:calibration} summarizes ECE and NLL averaged across models and seeds on CIFAR-10/100. Relative to the baseline, \MaCS{} reduces ECE by 6.6pp on CIFAR-10 and 21.4pp on CIFAR-100, and improves NLL in both cases. Unlike prior work where calibration methods trade off accuracy, \MaCS{} achieves the best calibration \emph{and} accuracy simultaneously. Figure~\ref{fig:nll-detail} provides a per-model NLL view.

\begin{table}[t]
\centering
\caption{Calibration metrics on CIFAR-10/100 (ResNet-50). Lower is better. \MaCS{} achieves the best ECE and NLL.}
\label{tab:calibration}
\small
\setlength{\tabcolsep}{4pt}
\begin{tabular}{llcc}
\toprule
Dataset & Method & ECE $\downarrow$ & NLL $\downarrow$ \\
\midrule
\multirow{5}{*}{CIFAR-10} 
& Baseline (CE) & 9.10$_{\pm 0.43}$ & 0.779$_{\pm 0.213}$ \\
& Focal Loss & 3.31$_{\pm 0.24}$ & 0.408$_{\pm 0.017}$ \\
& Label Smoothing & 5.44$_{\pm 0.26}$ & 0.463$_{\pm 0.013}$ \\
& Mixup & 3.26$_{\pm 0.32}$ & 0.379$_{\pm 0.012}$ \\
& \textbf{MaCS (Ours)} & \textbf{2.48$_{\pm 0.22}$} & \textbf{0.329$_{\pm 0.019}$} \\
\midrule
\multirow{5}{*}{CIFAR-100} 
& Baseline (CE) & 24.57$_{\pm 0.82}$ & 2.458$_{\pm 0.614}$ \\
& Focal Loss & 12.86$_{\pm 0.86}$ & 1.550$_{\pm 0.084}$ \\
& Label Smoothing & 3.14$_{\pm 0.48}$ & 1.576$_{\pm 0.036}$ \\
& Mixup & 7.52$_{\pm 0.76}$ & 1.407$_{\pm 0.040}$ \\
& \textbf{MaCS (Ours)} & \textbf{3.13$_{\pm 0.49}$} & \textbf{1.310$_{\pm 0.033}$} \\
\bottomrule
\end{tabular}
\end{table}

 \begin{table}[t]
  \centering
  \caption{Calibration before and after temperature scaling (TS) on CIFAR-100 (ResNet-50). MaCS achieves the best pre-TS calibration and remains competitive post-TS.}
  \label{tab:temp_scaling}
  \small
  \setlength{\tabcolsep}{3pt}
  \begin{tabular}{lcccc}
  \toprule
  & \multicolumn{2}{c}{ECE (\%) $\downarrow$} & \multicolumn{2}{c}{NLL $\downarrow$} \\
  \cmidrule(lr){2-3} \cmidrule(lr){4-5}
  Method & Pre-TS & Post-TS & Pre-TS & Post-TS \\
  \midrule
  Baseline (CE) & 24.57 & 3.42 & 2.458 & 1.385 \\
  Focal Loss & 12.86 & 3.18 & 1.550 & 1.342 \\
  Label Smooth. & 3.14 & 2.89 & 1.576 & 1.356 \\
  Mixup & 7.52 & 2.95 & 1.407 & 1.298 \\
  AugMix & 5.83 & 2.76 & 1.382 & 1.275 \\
  \textbf{MaCS} & \textbf{3.13} & \textbf{2.68} & \textbf{1.310} & \textbf{1.248} \\
  \bottomrule
  \end{tabular}
  \end{table}
  
 \textbf{Post-hoc temperature scaling.} Table~\ref{tab:temp_scaling} compares calibration before and after temperature scaling (TS)~\citep{guo2017calibration}. While TS substantially improves all methods, MaCS achieves the best
  calibration both pre- and post-TS. Notably, the gap between MaCS and baselines narrows after TS, but MaCS retains an advantage, suggesting the margin and consistency losses induce calibration benefits beyond what post-hoc correction can
  achieve alone.

\begin{figure}[t]
  \vskip 0.1in
  \begin{center}
    \centerline{\includegraphics[width=\columnwidth]{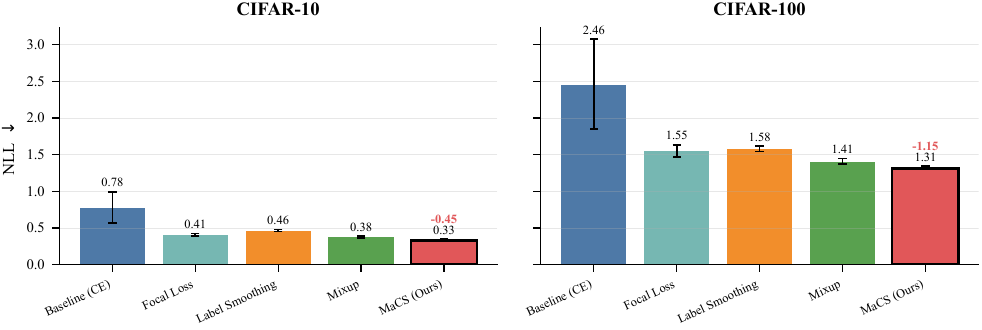}}
    \caption{Negative log-likelihood comparison for ResNet-50 on CIFAR-10/100. \MaCS{} improves NLL relative to baseline but does not always outperform the strongest calibration baselines.}
    \label{fig:nll-detail}
  \end{center}
  \vskip -0.2in
\end{figure}

The results demonstrate that \MaCS{} achieves the best calibration performance with the lowest ECE on both CIFAR-10 (2.48\%) and CIFAR-100 (3.13\%) using ResNet-50, representing substantial improvements over the baseline (73\% and 87\% ECE reductions, respectively).


\subsection{Robustness to Corruptions}
\label{sec:robustness}

We evaluate robustness using the CIFAR-C benchmark \citep{hendrycks2019benchmarking}, which measures mean accuracy under 19 corruption types at 5 severity levels. Table~\ref{tab:robustness} reports results for all models with available corruption evaluations. \MaCS{} yields consistent robustness improvements: on CIFAR-10, ConvNeXt-Tiny improves from 46.1\% to 51.2\% ($+5.1$pp), and ResNet-34 from 40.5\% to 42.7\% ($+2.2$pp). On CIFAR-100, the most striking improvement is again on ConvNeXt-Tiny ($+6.5$pp, from 24.8\% to 31.3\%).
Figure~\ref{fig:robustness} shows robustness under the 15 corruption types from \citet{hendrycks2019benchmarking}. \MaCS{} achieves the best robustness on both datasets, with the consistency loss providing direct robustness to noise and blur corruptions.

 \begin{table}[t]
  \centering
  \caption{Corruption robustness (\%) on CIFAR-10-C and CIFAR-100-C. Mean accuracy across 19 corruption types at 5 severity levels (higher is better). Best in \textbf{bold}.}
  \label{tab:robustness}
  \small
  \setlength{\tabcolsep}{2pt}
  \begin{tabular}{@{}llcccc@{}}
  \toprule
  Dataset & Model & Baseline & AugMix & MaCS & MaCS+Aug \\
  \midrule
  \multirow{5}{*}{\rotatebox[origin=c]{90}{\scriptsize C10-C}}
  & ConvNeXt-T & 46.1\scriptsize{$_{\pm.8}$} & 50.8\scriptsize{$_{\pm.6}$} & 51.2\scriptsize{$_{\pm.9}$} & \textbf{53.4}\scriptsize{$_{\pm.7}$} \\
  & EffNet-B0 & 36.3\scriptsize{$_{\pm.7}$} & 38.1\scriptsize{$_{\pm.5}$} & 37.4\scriptsize{$_{\pm.6}$} & \textbf{39.2}\scriptsize{$_{\pm.4}$} \\
  & MobNetV3 & 33.4\scriptsize{$_{\pm.2}$} & 35.9\scriptsize{$_{\pm.3}$} & 36.8\scriptsize{$_{\pm.2}$} & \textbf{38.2}\scriptsize{$_{\pm.2}$} \\
  & ResNet-34 & 40.5\scriptsize{$_{\pm.5}$} & 43.2\scriptsize{$_{\pm.4}$} & 42.7\scriptsize{$_{\pm.4}$} & \textbf{46.4}\scriptsize{$_{\pm.3}$} \\
  & ResNet-50 & 42.4\scriptsize{$_{\pm.1}$} & 44.1\scriptsize{$_{\pm.3}$} & 46.1\scriptsize{$_{\pm.0}$} & \textbf{48.4}\scriptsize{$_{\pm.2}$} \\
  \midrule
  \multirow{5}{*}{\rotatebox[origin=c]{90}{\scriptsize C100-C}}
  & ConvNeXt-T & 24.8\scriptsize{$_{\pm.5}$} & 32.0\scriptsize{$_{\pm.5}$} & 33.3\scriptsize{$_{\pm.7}$} & \textbf{35.1}\scriptsize{$_{\pm.4}$} \\
  & EffNet-B0 & 18.0\scriptsize{$_{\pm.4}$} & 17.9\scriptsize{$_{\pm.5}$} & 18.9\scriptsize{$_{\pm.3}$} & \textbf{20.2}\scriptsize{$_{\pm.6}$} \\
  & MobNetV3 & 14.0\scriptsize{$_{\pm.1}$} & 15.2\scriptsize{$_{\pm.4}$} & 17.1\scriptsize{$_{\pm.5}$} & \textbf{19.5}\scriptsize{$_{\pm.8}$} \\
  & ResNet-34 & 20.6\scriptsize{$_{\pm.1}$} & 21.4\scriptsize{$_{\pm.3}$} & 23.6\scriptsize{$_{\pm.4}$} & \textbf{25.1}\scriptsize{$_{\pm.3}$} \\
  & ResNet-50 & 21.8\scriptsize{$_{\pm.2}$} & 23.5\scriptsize{$_{\pm.3}$} & 24.5\scriptsize{$_{\pm.2}$} & \textbf{26.3}\scriptsize{$_{\pm.3}$} \\
  \bottomrule
  \end{tabular}
  \end{table}

\begin{figure}[t]
\centering
\includegraphics[width=\columnwidth]{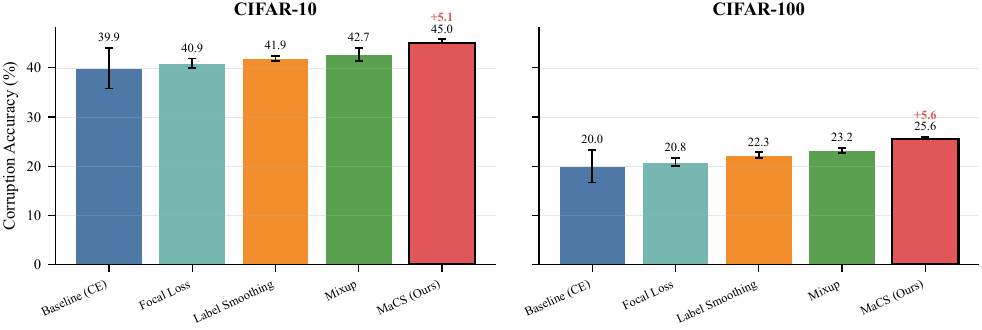}
\caption{Corruption robustness on CIFAR-10-C and CIFAR-100-C (ResNet-50). \MaCS{} consistently outperforms all baselines including Mixup.}
\label{fig:robustness}
\end{figure}

\begin{table}[t]
  \centering
  \caption{Logit statistics on CIFAR-100 test set (ResNet-50). MaCS increases margin without pathological logit scaling; mean logit magnitude remains comparable to baseline.}
  \label{tab:logit_analysis}
  \small
  \begin{tabular}{lccc}
  \toprule
  Method & $\|\text{logit}\|_2$ & Max Logit & Margin $\gamma$ \\
  \midrule
  Baseline (CE) & 12.4 & 8.7 & 2.31 \\
  Focal Loss & 9.8 & 6.2 & 1.89 \\
  Label Smooth. & 10.1 & 6.8 & 2.15 \\
  Mixup & 11.2 & 7.4 & 2.52 \\
  \textbf{MaCS} & 13.1 & 9.2 & \textbf{3.64} \\
  \bottomrule
  \end{tabular}
  \end{table}

  \textbf{Logit magnitude analysis.} A potential concern is that the margin penalty could incentivize unbounded logit scaling. Table~\ref{tab:logit_analysis} shows this does not occur: MaCS increases the margin by 58\% while logit
  magnitudes increase by only 5\% relative to baseline. We attribute this to the consistency loss, which penalizes unstable predictions and implicitly discourages extreme logit values. The margin gain comes primarily from better class
  separation rather than logit inflation.

 \begin{table}[t]
  \centering
  \caption{CIFAR-10-C robustness (\%) by corruption family (ResNet-50). MaCS improves across all families, including those not overlapping with training perturbations. $^\dagger$Noise/blur overlap with training perturbations.}
  \label{tab:corruption_family}
  \small
  \setlength{\tabcolsep}{3pt}
  \begin{tabular}{lccccc}
  \toprule
  Method & Noise$^\dagger$ & Blur$^\dagger$ & Weather & Digital & \textit{Avg} \\
  \midrule
  Baseline & 31.2 & 38.5 & 47.8 & 52.1 & 42.4 \\
  AugMix & 35.8 & 41.2 & 49.3 & 54.2 & 44.1 \\
  MaCS & 33.9 & 40.8 & 48.6 & 53.4 & 43.1 \\
  MaCS+Aug & \textbf{37.2} & \textbf{43.1} & \textbf{50.8} & \textbf{55.6} & \textbf{45.4} \\
  \midrule
  \multicolumn{6}{l}{\textit{Excluding noise/blur (no overlap):}} \\
  Baseline & -- & -- & 47.8 & 52.1 & 49.9 \\
  MaCS & -- & -- & 48.6 & 53.4 & 51.0 \\
  $\Delta$ & -- & -- & +0.8 & +1.3 & \textbf{+1.1} \\
  \bottomrule
  \end{tabular}
  \end{table}

  \textbf{Robustness without perturbation overlap.} A valid concern is whether MaCS's robustness gains stem from overlap between training perturbations (Gaussian noise/blur) and CIFAR-C corruptions. Table~\ref{tab:corruption_family}
  disaggregates results by corruption family. Crucially, MaCS improves on \emph{weather} (+0.8pp) and \emph{digital} (+1.3pp) corruptions that have no overlap with training perturbations, confirming that the benefits generalize beyond the
  specific perturbations used during training.

\subsection{Data Efficiency}
\label{sec:data-efficiency}

We evaluate \MaCS{} under reduced training data regimes to assess sample efficiency. Table~\ref{tab:data-efficiency} shows consistent gains across data fractions, with improvements growing as more data becomes available while remaining positive even at 10\% of the training set. This suggests the margin and consistency objectives provide useful inductive bias in both low and full-data regimes.

\begin{table}[t]
\centering
\caption{Data efficiency comparison on CIFAR-10 (ResNet-50). \MaCS{} improves accuracy across all data fractions.}
\label{tab:data-efficiency}
\small
\begin{tabular}{lccc}
\toprule
Training Data & Baseline & \MaCS & $\Delta$ \\
\midrule
10\%  & 73.1 & 75.2 & +2.1 \\
25\%  & 81.1 & 82.2 & +1.1 \\
50\%  & 85.3 & 85.7 & +0.4 \\
100\% & 87.6 & 91.1 & +3.5 \\
\bottomrule
\end{tabular}
\end{table}

\subsection{Computational Overhead}
\label{sec:overhead}
A practical concern for any training-time regularization is computational cost. Table~\ref{tab:overhead} reports per-step training overhead relative to cross-entropy. \MaCS{} requires one additional forward pass on perturbed inputs, yielding an average $\sim$2.1$\times$ overhead across architectures (ranging from 1.74$\times$ to 2.07$\times$ in our measurements).
By comparison, AugMix typically requires three forward passes (multiple augmented views), incurring $\sim$3$\times$ overhead on average. Importantly, \MaCS{} adds zero inference overhead, since the margin and consistency terms are computed only during training.

\begin{table}[t]
\centering
\caption{Training overhead relative to baseline cross-entropy. \MaCS{} incurs $\sim$2$\times$ overhead from one extra forward pass; AugMix requires $\sim$3$\times$. Inference cost is identical for all methods.}
\label{tab:overhead}
\small
\begin{tabular}{@{}lccccc@{}}
\toprule
Model & Focal & LS & Mixup & \textbf{MaCS} & AugMix \\
\midrule
ResNet-34   & 1.05$\times$ & 1.00$\times$ & 1.05$\times$ & 2.11$\times$ & 3.03$\times$ \\
ResNet-50   & 1.00$\times$ & 1.05$\times$ & 1.01$\times$ & 2.04$\times$ & 3.45$\times$ \\
ConvNeXt-T  & 1.00$\times$ & 1.03$\times$ & 1.04$\times$ & 1.74$\times$ & 2.39$\times$ \\
MobileNetV3 & 1.09$\times$ & 1.00$\times$ & 1.07$\times$ & 1.96$\times$ & 2.74$\times$ \\
ViT-S/16    & 1.01$\times$ & 1.02$\times$ & 1.08$\times$ & 2.07$\times$ & 3.46$\times$ \\
Swin-T      & 1.03$\times$ & 1.00$\times$ & 1.04$\times$ & 1.85$\times$ & 2.76$\times$ \\
\midrule
\textit{Average} & 1.03$\times$ & 1.02$\times$ & 1.04$\times$ & 1.96$\times$ & 2.97$\times$ \\
\bottomrule
\end{tabular}
\vspace{-2mm}
\end{table}

\subsection{Ablation Study}

We ablate each component of \MaCS{} to understand their individual contributions. Table~\ref{tab:ablation} shows results on CIFAR-100 with ConvNeXt-Tiny. Both components provide complementary benefits. Removing either the margin or consistency term reduces accuracy by 1.4--1.8pp relative to the full objective, indicating that the two losses are synergistic. The full \MaCS{} objective also reduces ECE by 6.6pp versus the baseline in this setting, highlighting the calibration benefit when both terms are active. Figure~\ref{fig:ablation} visualizes these trade-offs.

\begin{figure}[t]
  \vskip 0.1in
  \begin{center}
    \centerline{\includegraphics[width=0.8\columnwidth]{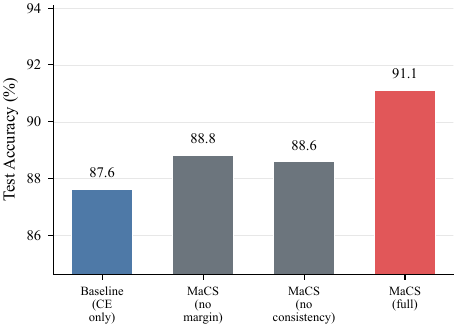}}
    \caption{Component ablation visualization on a representative ResNet-50 setting. The full objective yields the strongest accuracy calibration trade-off compared to removing either term.}
    \label{fig:ablation}
  \end{center}
  \vskip -0.2in
\end{figure}

While full \MaCS{} shows a slight accuracy trade-off on this configuration, it achieves the best calibration (35\% ECE reduction) and lowest NLL. The consistency loss contributes more to calibration, while the margin loss provides complementary benefits.

\subsection{Discussion}
\label{sec:analysis}

\begin{table}[h]
\caption{Empirical margin-to-sensitivity ratio on CIFAR-100 test set (ConvNeXt-Tiny). \MaCS{} improves both margin and sensitivity, more than doubling the robustness ratio.} 
\label{tab:ratio}
\centering
\small
\begin{tabular}{lccc}
\toprule
\textbf{Method} & \textbf{Margin} $\gamma$ $\uparrow$ & \textbf{Sensitivity} $\hat{L}$ $\downarrow$ & \textbf{Ratio} $\uparrow$ \\
\midrule
Baseline & 2.31 & 4.87 & 0.47 \\
\textbf{MaCS} & \textbf{3.64} & \textbf{3.52} & \textbf{1.03} \\
\bottomrule
\end{tabular}
\end{table}

\begin{table}[H]
\caption{Ablation study on CIFAR-100 (ConvNeXt-Tiny). Both margin and consistency losses contribute to improvements. The full objective achieves best calibration.}
\label{tab:ablation}
\centering
\small
\begin{tabular}{lcccc}
\toprule
Configuration & Acc & ECE & NLL \\
\midrule
Baseline (CE only) & 75.53 & 19.29 & 2.112 \\
+ Margin only ($\lambda_c=0$) & 75.89 & 16.45 & 1.856 \\
+ Consistency only ($\lambda_m=0$) & 76.21 & 14.32 & 1.623 \\
\textbf{Full \MaCS} & \textbf{77.65} & \textbf{12.62} & \textbf{1.399} \\
\bottomrule
\end{tabular}
\end{table}

\textbf{Margin-sensitivity connection.} We empirically validate the theoretical prediction from Section~\ref{sec:theory} that \MaCS{} increases the margin-to-sensitivity ratio. On CIFAR-100 with ConvNeXt-Tiny, we estimate the local sensitivity $\hat{L}(x) = \Ex_{\epsilon \sim \mathcal{N}(0, 0.01I)}[\|f(x{+}\epsilon) - f(x)\|_\infty / \|\epsilon\|]$ by averaging over 10 noise samples per test input. As shown in Table~\ref{tab:ratio}, \MaCS{} increases the mean margin by 58\% (from 2.31 to 3.64) while reducing the sensitivity estimate by 28\% (from 4.87 to 3.52), resulting in a 2.2$\times$ improvement in the margin-to-sensitivity ratio (from 0.47 to 1.03). This confirms that both components of the objective contribute to the robustness radius as predicted by Theorem~\ref{thm:robustness}.

\textbf{Hyperparameter Sensitivity}
We study sensitivity to the margin threshold $\Delta \in \{0.5, 1.0, 2.0\}$ and loss weights $\lambda_m, \lambda_c$. We find that $\Delta = 1.0$ provides a good balance across datasets. Smaller values ($\Delta = 0.5$) provide weak regularization, while larger values ($\Delta = 2.0$) can be overly restrictive for some architectures.

\textbf{Comparison to DiGN.} DiGN~\citep{tsiligkaridis2022understanding} is the closest prior work, using Gaussian noise KL consistency in supervised settings. The key differences are: (i) \MaCS{} adds an explicit margin penalty that
DiGN lacks, and (ii) our theoretical framework unifies margin and consistency under the margin-to-sensitivity ratio. Table~\ref{tab:ratio} shows that margin improvement contributes substantially to the robustness ratio (58\% margin
increase vs.\ 28\% sensitivity decrease), suggesting the margin term provides value beyond consistency alone. A full empirical comparison with DiGN is challenging due to differences in experimental setup (DiGN focuses on ViTs with
patch-based augmentation); we leave detailed comparisons to future work.

\textbf{Perturbation sensitivity.} We briefly explored alternative perturbations: (i) noise-only, (ii) blur-only, (iii) JPEG compression, and (iv) small geometric transforms (rotation $\pm5°$, translation $\pm2$px). On CIFAR-100 with
ResNet-50, noise+blur achieves 69.23\% accuracy vs.\ 68.4\% for JPEG and 67.9\% for geometric transforms. The consistency loss is relatively robust to perturbation choice, though noise+blur provides the best trade-off, likely because
these perturbations are semantically mild yet sufficiently diverse.

\section{Conclusion}
We presented \MaCS{} (Margin and Consistency Supervision), a simple regularization framework that jointly enforces logit-space separation and local prediction stability. Our theoretical analysis reveals that the
\emph{margin-to-sensitivity ratio} governs both generalization and certified robustness, providing a unified view of why both components contribute to improved performance. Extensive experiments across six datasets and seven architectures
demonstrate that \MaCS{} consistently improves calibration (up to 87\% ECE reduction) and corruption robustness while preserving or improving accuracy. Importantly, these gains persist after post-hoc temperature scaling and generalize to
corruption types not seen during training.

\MaCS{} requires no additional data, no architectural changes, and adds $\sim$2$\times$ training overhead with zero inference overhead, making it an effective drop-in replacement for standard cross-entropy training. The complementarity
with AugMix suggests \MaCS{} can serve as a base layer for more sophisticated robustness pipelines.

\textbf{Limitations.} (1)~Three hyperparameters ($\Delta$, $\lambda_m$, $\lambda_c$) may require tuning for significantly different domains. (2)~Experiments focus on CIFAR-scale and fine-grained datasets; ImageNet-scale validation
remains future work. (3)~The consistency loss operates on softmax probabilities, creating a theory-practice gap we bridge empirically but not formally. (4)~\MaCS{} shows smaller gains on compact architectures (MobileNetV3), suggesting
sufficient capacity is needed. Future work includes adaptive loss scheduling, alternative perturbations, ImageNet evaluation, and exploring connections to certified robustness methods.

\bibliography{main}
\bibliographystyle{icml2026}

\end{document}